\documentclass{article}
\usepackage{log_2023}            

\usepackage{booktabs}            
\usepackage{multirow}            
\usepackage{amsfonts}            
\usepackage{graphicx}            

\usepackage{subfigure}
\usepackage{amsmath}
\usepackage{amssymb}
\usepackage{mathtools}
\usepackage{amsthm}
\usepackage{blkarray}
\usepackage{tabularx}

\usepackage[numbers,compress,sort]{natbib}
\usepackage{tikz}
\usetikzlibrary{decorations.pathreplacing}
\usetikzlibrary{arrows, decorations.pathmorphing}
\usetikzlibrary{math}

\definecolor{cpar}{rgb}{0.2,0.7,0.7}
\definecolor{cseq}{RGB}{247,149,72}

\newcommand{\N}{\mathcal{N}}
\newcommand{\G}{\mathcal{G}}
\newcommand{\SA}{\mathcal{S}}
\newcommand{\PA}{\mathcal{P}}
\newcommand{\A}{\mathcal{A}}

\theoremstyle{definition}
\newtheorem{definition}{Definition}
\newtheorem{corollary}{Corollary}
\newtheorem{assumption}{Assumption}
\theoremstyle{remark}

\title[Parallel Algorithms Align with Neural Execution]{Parallel Algorithms Align with Neural Execution}

\author[V. Engelmayer et al.]{%
Valerie Engelmayer\thanks{Corresponding author.}\\
University of Augsburg\\
\email{valerie.engelmayer@gmail.com}\And
Dobrik Georgiev\\ 
University of Cambridge\\
\email{dgg30@cam.ac.uk}\\
\And
Petar Veličković\\ 
Google DeepMind\\
\email{petarv@google.com}\\
}

\begin{document}

\maketitle

\begin{abstract}

Neural algorithmic reasoners are parallel processors. Teaching them sequential algorithms contradicts this nature, rendering a significant share of their computations redundant. Parallel algorithms however may exploit their full computational power, therefore requiring fewer layers to be executed. This drastically reduces training times, as we observe when comparing parallel implementations of searching, sorting and finding strongly connected components to their sequential counterparts on the CLRS framework. Additionally, parallel versions achieve (often strongly) superior predictive performance.
\end{abstract}

\section{Motivation}

Classical algorithms often pose a bottleneck to information processing \cite{velickovic_neural_2021}. They are usually designed to deal with consistent, totally ordered, abstract quantities, while in reality, we need to reason about noisy, high-dimensional data. Machine learning and neural networks (NNs) in particular enable machines to extract useful features from such inputs, but if their outputs need to be composed with a non-differentiable algorithm, they cannot learn from direct feedback via backpropagation. Moreover, compressing information in a way that makes the algorithm applicable loses a lot of potentially relevant detail. Breaking this bottleneck by teaching NNs how to execute algorithms is the objective of \emph{neural algorithmic reasoning} \cite{yan_neural_2020, velickovic_neural_2021, velickovic_clrs_2022}. First applications to real-world data are promising \cite{vrvcek2020step, velickovic_reasoning-modulated_2022, georgievnarti}, but extrapolation still has room for improvement even on highly elaborate architectures \cite{ibarz_generalist_2022, bevilacqua_neural_2023}. Therefore, there is a clear need to investigate neural networks' information processing capabilities more closely.

When executing algorithms, NNs act as computational machines. In graph neural networks (GNNs), graph nodes take on the role of storage space (interpreting edge labels as nodes adjacent to its endpoints throughout this paper), while edges indicate which ways information may flow. The update function of choice defines the set of constant (neural) time operations. But note how nodes update their features \emph{in parallel}, each one acting as a processor of its own rather than sheer memory.

The parallel nature of neural networks is widely known. Running them in parallel fashion on processing devices like GPUs and TPUs drastically saves computational resources \cite{zhang2019recent, yazdanbakhsh2021evaluation}. It seems natural that this translation between computational models would also hold the other way around. And indeed, \citet{loukas_what_2020} proves how GNNs are analogous to distributed computational models under certain assumptions. 
\citet{kaiser2015neural} exploit the advantages of parallel processing in their Neural GPU. The differentiable sorting algorithms by \citet{petersen_differentiable_2021} operate in parallel. \citet{freivalds_neural_nodate} derive their architecture from the parallel computational model of Shuffle-Exchange Networks. \citet{xu_what_2020} observe how their model learns to compute a shortest path starting from both ends in parallel when executing the Bellman-Ford algorithm. \citet{velickovic_clrs_2022} and \citet{velickovic_neural_2020} hint at the favourability of using parallelized computations whenever possible.

It is time the parallel processing capabilities of NNs are exploited systematically, and this paper takes a relevant step in that direction. Theory on parallel computational models and algorithms explicitly designed for them are abundant \cite{gibbons_efficient_1990, greenlaw_limits_1995, parhami2006introduction}. Their trajectories are shorter and align more closely with neural architectures, as illustrated in Figure \ref{fig:highlevel}. Hinting at these during training teaches NN to execute algorithmic tasks much more efficiently than when providing hints for sequential algorithms, as we demonstrate in Section \ref{sec:exp} for the examples of searching, sorting and finding strongly connected components. While it is common practice to modify the neural architecture for better alignment \cite{xu_what_2020, velickovic_pointer_2020, georgiev_algorithmic_2022, ibarz_generalist_2022, mahdavi_towards_2023}, it seems promising to narrow the gap from the other side, by choosing algorithms that naturally align with neural execution.
\begin{figure}
\centering
\subfigure[Neural processing is highly parallel.]{
\begin{tikzpicture}[scale=0.5]
\node at (-2,2){};
\node at (8,2){};

\foreach \y in {0,...,4}
	\foreach \x  in {0,2,4,6}
		{
		\node (\x\y) [draw,circle,fill=gray!50,outer sep=1pt] at (\x,\y) {};
		}
\foreach \y [count = \yi] in {0,...,4}{
		\draw [->] (0\y) -- (2\y);
		\draw [->] (2\y) -- (4\y);
		\draw [->] (4\y) -- (6\y);
		}
\foreach \y [count = \yi] in {0,...,3}{
		\draw [->] (0\y) -- (2\yi);
		\draw [->] (2\y) -- (4\yi);
		\draw [->] (4\y) -- (6\yi);
		\draw [->] (0\yi) -- (2\y);
		\draw [->] (2\yi) -- (4\y);
		\draw [->] (4\yi) -- (6\y);
		}		
\end{tikzpicture}
}\vfill
\subfigure[Trajectories of sequential algorithms are sparse.]{
\begin{tikzpicture}[scale =0.4]

\foreach \y [count = \yi] in {0,...,4}
	\foreach \x [count = \xi] in {0,3,6,9,12,15,18}
		{
		\draw [fill=gray!50] (\x,\y) rectangle (\x+1,\yi);
		}
		\draw [fill=cseq] (0,2) rectangle (1,3);
		\draw [fill=cseq] (0,0) rectangle (1,1);
		
		\draw [fill=cseq] (3,3) rectangle (4,4);
		\draw [fill=cseq] (3,1) rectangle (4,2);
		
		\draw [fill=cseq] (6,3) rectangle (7,4);
		\draw [fill=cseq] (6,2) rectangle (7,3);
		
		\draw [fill=cseq] (9,2) rectangle (10,3);
		\draw [fill=cseq] (9,1) rectangle (10,2);
		
		\draw [fill=cseq] (12,3) rectangle (13,4);
		\draw [fill=cseq] (12,1) rectangle (13,2);
		
		\draw [fill=cseq] (15,4) rectangle (16,5);
		\draw [fill=cseq] (15,0) rectangle (16,1);
		
		\draw [fill=cseq] (18,2) rectangle (19,3);
		
\draw [->,cseq,thick] (1.1,2.5)--(2.9,1.5);
\draw [->,cseq,thick] (1.1,0.5)--(2.9,1.3);
\draw [->,cseq,thick] (4.1,1.5)--(5.9,3.3);
\draw [->,cseq,thick] (4.1,3.5)--(5.9,3.5);
\draw [->,cseq,thick] (7.1,3.5)--(8.9,2.7);
\draw [->,cseq,thick] (7.1,2.5)--(8.9,2.5);

\draw [->,cseq,thick] (10.1,2.5)--(11.9,1.7);
\draw [->,cseq,thick] (10.1,1.5)--(11.9,1.5);
\draw [->,cseq,thick] (13.1,1.5)--(14.9,4.4);
\draw [->,cseq,thick] (13.1,3.5)--(14.9,4.6);
\draw [->,cseq,thick] (16.1,4.5)--(17.9,2.6);
\draw [->,cseq,thick] (16.1,0.5)--(17.9,2.4);
\end{tikzpicture}
}\hfill
\subfigure[Trajectories of parallel algorithms are denser and shorter.]{
\begin{tikzpicture}[scale =0.4]
\node at (-2,0){}; 
\node at (11,0){}; 

\foreach \y [count = \yi] in {0,...,4}
	\foreach \x [count = \xi] in {0,3,6,9}
		{
		\draw [fill=cpar, rounded corners=0.2ex] (\x,\y) rectangle (\x+0.8,\y + 0.8);
		}
		
\foreach \y in {0,...,4}
	\foreach \x in {0,3,6}
		{
		\draw [->,cpar, thick] (\x+0.9,\y+0.4)--(\x+2.9,\y+0.4);
		}
		
\foreach \y in {0,...,3}
	\foreach \x in {0,6}
		{
		\draw [->,cpar, thick] (\x+0.9,\y+0.4)--(\x+2.9,\y+1.2);
		}
		
\foreach \y in {1,...,4}
	\foreach \x in {3}
		{
		\draw [->,cpar, thick] (\x+0.9,\y+0.4)--(\x+2.9,\y-0.4);
		}
\end{tikzpicture}
}

\caption{Trajectories of sequential and parallel algorithms, as well as neural processing.}
\label{fig:highlevel}

\end{figure}
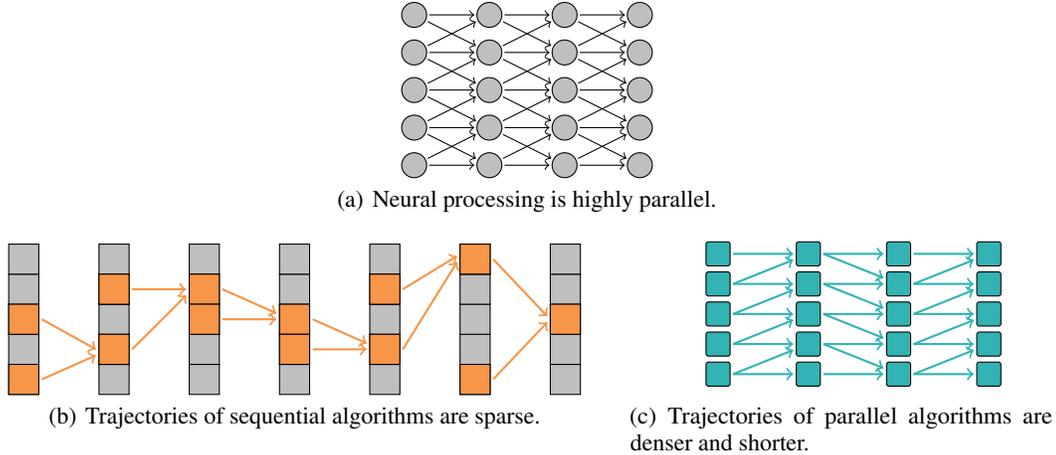

\section{Parallel Computing}

Fundamentally, the parallel computational models addressed here assume multiple processors collaborating to solve a task. The line between parallel and distributed computing is blurry and depends on how controlled the interactions between processors are. We assume a fixed and known interconnection graph, uniquely identified processors and a common clock to govern computation. Therefore, we choose to speak of parallel computing.

\subsection{Parallel Computational Models}

\paragraph{Processor Arrays.} Communication may take place via hard-wired channels between the processors. These induce an interconnection graph that may in principle take any shape. At every time step, each processor executes some computation based on the contents of its local memory and the information received from its neighbours in the previous step, and may, in turn, send out a tailored message through any of its channels.

\paragraph{PRAM Models.} Alternatively, communication may be realised by reading from and writing to global memory, giving rise to \emph{PRAM} (parallel random access machine) models \cite{gibbons_efficient_1990}. Submodels allowing for \emph{concurrent} reading and writing by multiple processors are referred to as \emph{CRCW} PRAM. Different conventions exist on whether attempting to concurrently write different values is permitted, and if so, how to decide who succeeds. In the most powerful model, the \emph{priority} CRCW PRAM, the value from the processor with the lowest index taking part in the concurrent write will be taken on. 

\subsection{Efficiency}
\label{sec:eff}
Since multiple steps can be carried out at the same time, the required number of operations in a parallel algorithm does not impose a lower bound to its run time as in the sequential case, but the product of time and processor number. \emph{Optimal speedup} is achieved if the use of $n$ processors speeds up computation by a factor of $n$.
This gives rise to a notion of efficiency frequently used in parallel computing \cite{gibbons_efficient_1990}.
\begin{definition}
    The \emph{efficiency} of a parallel algorithm solving a task of sequential complexity $C$ on $p$ processors in time $t$
    is defined as
    \begin{align*}
        \frac{C}{pt}.
    \end{align*}
\end{definition}
It is not hard to see that optimal speedup entails an efficiency of $\Omega(1)$.

\subsection{Examples of Parallel Algorithms}
\paragraph{Searching.}
\label{sec:algo}
For a simple parallel search for value $x$ in a descending list of $n$ items, assume a priority CRCW PRAM with $n$ processors. Distribute the first item to processor 1, the second to processor 2 etc., while $x$ is stored in the global memory. If a processor's item is $\geq x$, it tries to write its index to a designated location in the global memory. Since the one with the smallest index will succeed, the location now contains the desired position of $x$.
The run time is independent of the input size\footnote{Distributing values to processors can be done in constant time by routing over the shared memory. We neglect distributing/returning in-/outputs from/to a host computer in the following as it is omitted in neural execution.}, so the time-processor-product is $\Theta(n)$, missing optimal speed-up as sequential searching can be done in $O(\log n)$.

\paragraph{Sorting.} \citet{habermann_parallel_1972} proposes a simple parallel sorting algorithm for a linear array of processors called Odd-Even Transposition Sort (OETS). Each processor holds one item. In an odd (even) round, all neighbouring pairs starting at an odd (even) index swap their items if they are out of order. The two types of rounds take turns for at most $n$ rounds total when $n$ items are to be sorted, yielding $O(n^2)$ operations when accounting for the $n$ processors. Again, this is not optimal for comparison-based sorting, which may be done in $O(n \log n)$.

\paragraph{Strongly Connected Components.}
\citet{rolim_identifying_2000} propose a Divide-and-Conquer algorithm for computing strongly connected components (SCC) of a digraph, which they call DCSC. First, find all descendants and predecessors of an arbitrary node, e.g. by carrying out a breadth-first search (BFS) in the graph and its reversed version. The intersection of both sets constitutes a SCC. Observe how each further SCC has to be completely contained in either the descendants, the predecessors or the undiscovered nodes, such that the described routine may be called recursively for start nodes in each subset independently until each vertex is assigned to a SCC. They prove an expected serial time complexity of $O(n \log n)$ for graphs on $n$ nodes whose degrees are bounded by a constant. This is not optimal, but parallelization of the two searches per vertex, as well as the recursive calls, may significantly speed up execution. 

\subsection{Analogy to Neural Networks}
 \citet{loukas_what_2020} formally establishes an analogy between models like processor arrays and GNN by identifying processors with graph nodes and communication channels with edges. Therefore, the width of a GNN corresponds to $p$, and its depth to $t$. Loukas coins the term \emph{capacity} for the product of width and depth of a GNN, reflecting the time-processor product of parallel algorithms. The shared memory of a PRAM finds its neural analogue in graph-level features. Since the computation of a graph feature may take into account positional encodings of the nodes, we may assume a priority CRCW PRAM, encompassing all other PRAM models. 

\section{Efficiency of Executing Algorithms Neurally} 
Inspired by the definition of efficiency in parallel computing, we define the efficiency of a neural executioner as follows. 

 \begin{definition}
     Let $\G$ be a GNN with capacity $c(n)$ executing an algorithm $\A$ of sequential complexity $C(n)$. Define its \emph{node efficiency} as 
     \begin{align*}
         \eta (\G, \A) \coloneqq \frac{C(n)}{c(n)}.
     \end{align*}
 \end{definition}

This definition implies an important assumption we make throughout this paper. 
\begin{assumption}
When executing an algorithm on a GNN, one constant-time operation is to be executed per node per layer.
\end{assumption}
This is not entirely unproblematic as discussed in section \ref{sec:diss}, but often expected when providing hints and helps to identify theoretical properties. Under this assumption, node efficiency denotes the share of nodes doing useful computations throughout the layers.
 Since the computational cost of a GNN also scales with the number of messages that are being sent, it is insightful to study the share of edges that transport relevant information as well. 
 
\begin{definition}
    Let $\G$ be a GNN operating over a graph $G=(V,E)$, $m \coloneqq \vert E \vert$, to execute an algorithm $\A$. Then we call an edge $(i,j) \in E$ \emph{active} at layer $t$ for a certain input $x$, if the operation to be executed by node $j$ at time $t$ involves information stored at node $i$ at time $t-1$.
    Let $a(t)$ be the number of active edges at time $t$, and $T$ the total number of time-steps.
    Then define \emph{edge efficiency} as worst case share of active edges when processing inputs $x_n$ of size $n$,
    \begin{align*}
         \epsilon (\G, \A) \coloneqq \underset{x_n}{\min \ } \frac{1}{T} \sum_{t=1}^T \frac{a(t)}{m}.
    \end{align*}
\end{definition}

Note how neural efficiencies are defined relative to the \emph{algorithm} they are executing as opposed to the \emph{task} they solve. This allows for a neural executioner to be efficient in executing an algorithm that is itself not efficient in solving a task.

\subsection{Parallel Algorithms Entail Higher Efficiency}
\label{sec:theory}

Contradicting a GNN's parallel nature by teaching it to execute sequential algorithms artificially impedes the task. Training to solve tasks in parallel instead is more efficient, which may also simplify the function to learn. 

\paragraph{Shorter Trajectories.}
As observed by \citet{loukas_what_2020}, the complexity of an algorithm lower bounds the capacity of a GNN executing it. If the number of processors is one, the depth alone needs to match the complexity, while the width might theoretically be set to one. But in practice, the width has to scale with the input size $n$ to ensure applicability to different $n$. 
Therefore, \emph{training sequential algorithms forces overspending on capacity by a factor of $n$}.

Setting the width to $n$, as is often done to distribute one unit of information over each node, entails $n$ available processors. Making use of them may shorten the trajectory of an algorithm by a factor of up to $n$ in the case of optimal speedup, which allows the capacity to take on its lower bound. 
The capacity of a GNN directly translates to the time needed to train and execute it. Additionally, long roll-outs give rise to an issue \citet{bansal_end--end_2022} refer to as \emph{overthinking}, where many iterations degenerate the behaviour of a recurrent processor.

\paragraph{Less Redundancy.}
Neural efficiencies denote the share of nodes and edges involved in useful computations. Redundant computations not only harm run times but may also interfere with the algorithmic trajectory. Parameterising them correctly to prevent this can complicate the function to learn.
Assuming the redundant nodes (grey in figure \ref{subfig:seq}) need to preserve their information to be processed or put out later, their self-edges should execute an identity, while the additional incoming messages need to be ignored, i.e. mapped to a constant. 
In practice, this will be hard to do, which could entail a temporal variant of oversmoothing, where relevant information gets lost throughout the layers \cite{skipconnections}. \citet{skipconnections} highlight how skip connections help to avoid the issue, \citet{ibarz_generalist_2022} introduce a gating mechanism to leave information unchanged, \citet{bansal_end--end_2022} let their architecture recall the original input. 

So let's explore the efficiency of executing sequential and parallel algorithms.\\

\begin{corollary}
Let $\G$ be a scalable GNN operating over a graph with $n$ nodes and $m$ edges. Further let $\SA$ be a sequential, and $\PA$ an efficient parallel algorithm on $n$ processors, both of complexity $C$. Then executing $\SA$ and $\PA$ on $\G$, respectively, entails efficiencies
\begin{align*}
    & \eta(\G, \SA) = O \bigg(\frac{1}{n}\bigg),\ \ & \epsilon(\G, \SA) = O\bigg( \frac{1}{m}\bigg),\\
    & \eta(\G, \PA) = O(1),\ \ & \epsilon(\G, \PA) = O\bigg(\frac{n}{m}\bigg).
\end{align*}
\end{corollary}
\begin{proof}
As observed above, the capacity $c$ of a GNN executing a sequential algorithm of complexity $C$ has to be $c \geq nC$, while it may be $c=C$ in the case of optimal speedup. Node efficiencies follow immediately. 
Since one processor can read only so much information, only a constant number of edges can be active at each layer during sequential processing, while up to a multiple of $n$ edges can be active during parallel algorithms. This yields the stated edge efficiencies.
\end{proof}
Therefore, the share of nodes avoiding redundant computation cannot exceed $1/n$ when executing sequential algorithms, whereas it may reach up to 1 for efficient parallel algorithms. At the same time, the number of redundant messages is reduced by a factor of $n$. Removing the artificial bottleneck of a single processor prevents data from having to be stored until the processor gets to it. Allowing nodes to carry out meaningful computation frees them of the dead weight of acting as memory. 

\paragraph{Local Exchange of Information.}
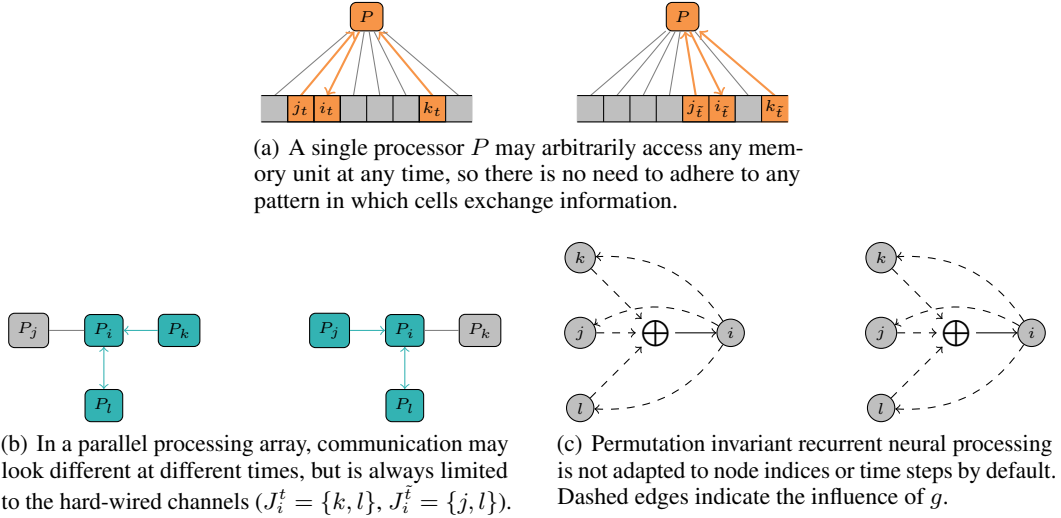
\begin{figure}
    \centering
    
    \subfigure[A single processor $P$ may arbitrarily access any memory unit at any time, so there is no need to adhere to any pattern in which cells exchange information.]{
        \begin{tikzpicture}[scale = 0.7]
            \label{subfig:1P}
            \node (P) [draw,rectangle,fill=cseq, rounded corners=0.5ex] at (2, 2) {\tiny \(P\)};
            \draw [fill=gray!50,draw=gray!50] (0,0) rectangle (0.5,0.5);
            \draw [fill=cseq] (0.5,0) rectangle (1,0.5);
            \draw [gray] (0.25,0.5) -- (P);
            \draw [->, cseq, thick] (0.75,0.5) -- (P);
            \draw [<-, cseq, thick] (1.25,0.5) -- (P);
            \draw [fill=cseq] (1,0) rectangle (1.5,0.5);
            \draw [gray] (1.75,0.5) -- (P);
            \draw [gray] (2.25,0.5) -- (P);
            \draw [gray] (2.75,0.5) -- (P);
            \draw [fill=gray!50] (1.5,0) rectangle (3,0.5);
            \draw [->, cseq, thick] (3.25,0.5) -- (P);
            \draw [fill=cseq] (3,0) rectangle (3.5,0.5);
            \draw [fill=gray!50,draw=gray!50] (3.5,0) rectangle (4,0.5);
            \draw [gray] (3.75,0.5) -- (P);

            \node at (1.25, 0.25){\tiny \(i_t\)};
            \node at (0.75, 0.25){\tiny \(j_t\)};
            \node at (3.25, 0.25){\tiny \(k_t\)};
            
            \draw (0,0) -- (4,0);
            \draw (0,0.5) -- (4,0.5);
            
            \foreach \x in {0.5,1,1.5,2,2.5,3,3.5}{
                \draw (\x,0) -- (\x,0.5);
            }


            \node (P) [draw,rectangle,fill=cseq, rounded corners=0.5ex] at (8, 2) {\tiny \(P\)};
            \draw [fill=gray!50,draw=gray!50] (6,0) rectangle (7,0.5);
            \draw [gray] (6.25,0.5) -- (P);
            \draw [gray] (6.75,0.5) -- (P);
            \draw [gray] (7.25,0.5) -- (P);
            \draw [fill=gray!50,draw=gray!50] (7,0) rectangle (7.5,0.5);
            \draw [gray] (7.75,0.5) -- (P);
            \draw [fill=cseq] (8,0) rectangle (9,0.5);
            \draw [->, cseq, thick] (8.25,0.5) -- (P);
            \draw [<-, cseq, thick] (8.75,0.5) -- (P);
            \draw [fill=gray!50,draw=gray!50] (7.5,0) rectangle (8,0.5);
            \draw [gray] (9.25,0.5) -- (P);
            \draw [fill=gray!50,draw=gray!50] (9,0) rectangle (9.5,0.5);
            \draw [fill=cseq, draw=cseq] (9.5,0) rectangle (10,0.5);
            \draw [->, cseq, thick] (9.75,0.5) -- (P);
            
            \foreach \x in {6.5,7,7.5,8,8.5,9,9.5}{
                \draw (\x,0) -- (\x,0.5);
            }

            \node at (8.75, 0.25){\tiny \(i_{\Tilde{t}}\)};
            \node at (8.25, 0.25){\tiny \(j_{\Tilde{t}}\)};
            \node at (9.75, 0.25){\tiny \(k_{\Tilde{t}}\)};
            
            \draw (6,0) -- (10,0);
            \draw (6,0.5) -- (10,0.5);
            
        \end{tikzpicture}
    }\hfill
    \subfigure[In a parallel processing array, communication may look different at different times, but is always limited to the hard-wired channels ($J_i^t = \{ k,l\}$, $J_i^{\Tilde{t}} = \{ j,l\}$).]{
    \begin{tikzpicture}[scale=0.6]
        \label{subfig:PP}
        \node (P) [draw,fill=cpar, rounded corners=0.5ex]{\tiny \(P_i\)};
        \node (j) [draw,left of=P, fill=gray!50, rounded corners=0.5ex] {\tiny \(P_j\)};
        \draw[gray] (j) --(P);
        
        \node (k) [draw,right of=P, fill=cpar, rounded corners=0.5ex] {\tiny \(P_k\)};
        \draw[->, cpar] (k) --(P);
        
        \node (l) [draw,below of=P, fill=cpar, rounded corners=0.5ex] {\tiny \(P_l\)};
        \draw[<->, cpar] (l) --(P);
        
        \node (space) [right of=k]{};
        \node (j) [draw,right of=space, fill=cpar, rounded corners=0.5ex] {\tiny \(P_j\)};
        \node (P2) [draw,right of=j, fill=cpar, rounded corners=0.5ex]{\tiny \(P_i\)};
        \draw[->, cpar] (j) --(P2);
        
        \node (k) [draw,right of=P2, fill=gray!50, rounded corners=0.5ex] {\tiny \(P_k\)};
        \draw[gray] (k) --(P2);
        
        \node (l) [draw,below of=P2, fill=cpar, rounded corners=0.5ex] {\tiny \(P_l\)};
        \draw[<->, cpar] (l) --(P2);
    \end{tikzpicture}
    }\hfill
    \subfigure[Permutation invariant recurrent neural processing is not adapted to node indices or time steps by default. Dashed edges indicate the influence of $g$.]{ 
        \begin{tikzpicture}[scale=0.6, Dot/.style={draw,circle,fill=gray!50,inner sep=2pt}] 
        \label{subfig:NN}
            \node (i) [Dot] {\tiny \(i\)};
            \node (agg) [black,left of=i, circle, inner sep=0, outer sep=0] {\( \bigoplus\)};
            \draw[->] (agg) -- (i);
            \node [left of=agg, Dot] (j) {\tiny \(j\)};
            \draw[->, dashed] (j) -- (agg);
            \node [above of=j, Dot] (k) {\tiny \(k\)};
            \draw[->, dashed] (k) -- (agg);
            \node [below of=j, Dot] (l) {\tiny \(l\)};
            \draw[->, dashed] (l) -- (agg);
            \draw [->, bend right, dashed] (i) edge (k);
            \draw [->, bend left, dashed] (i) edge (l);
            \draw [->, bend right, dashed] (i) edge (j);

            \node (space)[right of=i]{};
            \node [right of=space, Dot] (j) {\tiny \(j\)};
            \node (agg) [right of=j, circle, inner sep=0, outer sep=0] {\( \bigoplus\)};
            \node (i) [right of = agg, Dot] {\tiny \(i\)};
            \draw[->] (agg) -- (i);
            \draw[->, dashed] (j) -- (agg);
            \node [above of=j, Dot] (k) {\tiny \(k\)};
            \draw[->, dashed] (k) -- (agg);
            \node [below of=j, Dot] (l) {\tiny \(l\)};
            \draw[->, dashed] (l) -- (agg);
            \draw [->, bend right, dashed] (i) edge (k); 
            \draw [->, bend left, dashed] (i) edge (l);
            \draw [->, bend right, dashed] (i) edge (j);
        \end{tikzpicture}
    }
    \caption{Local view on information flow in different computational models at two different time steps $t$ and $\Tilde{t}$.}
    \label{fig:my_label}
\end{figure}
In neural networks, information exchange is inherently local. The feature $h_i^t$ of node $i$ at time $t$ may only depend on itself and its neighbours $\N_i$. E.g. for permutation invariant MPNN \cite{gilmer2017neural}, 
\begin{align}
    h_i^t = f \big(h_i^{t-1}, \underset{j \in \N_i}{\bigoplus} g(h_i^{t-1}, h_j^{t-1})\big)
    \label{eq:NN}
\end{align}
This paradigm is often not respected by classical algorithms, as depicted in figure \ref{subfig:1P}.
In the RAM model, the state $h_{i_t}^t$ of register $i_t$ updated at time $t$ may depend on any two registers $j_t$ and $k_t$:
\begin{align}
    h_{i_t}^t = f^t_i (h_{k_t}^{t-1}, h_{j_t}^{t-1}), \text{  $j_t, k_t$ arbitrary.} 
    \label{eq:1P}
\end{align}
Not being able to restrict which nodes have to communicate may render it advisable for a GNN to operate over a complete graph to make sure all necessary information is available at all times (see e.g. \cite{ibarz_generalist_2022}). The situation is different in the setting of interconnected processing arrays, see figure \ref{subfig:PP}. For example, OETS only ever requires neighbouring processors to compare their items. In general, at time $t$, the memory state $h_i^t$ of processor $i$ is computed by
\begin{align}
    h_i^t = f^t_i (h_i^{t-1}, \underset{j \in J_i^t}{\vert \vert} h_j^{t-1}),\ J_i^t \subseteq \N_i,
    \label{eq:PP}
\end{align}
where concatenation indicates how $i$ may tell apart its neighbours.
 Therefore it suffices for the GNN to only rely on edges present in the interconnection graph. To emulate a PRAM algorithm, an empty graph would in principle be enough, though it might not deem advantageous to route all communication over the graph feature in practice.
 Restricting the number of edges further reduces the use of resources and may help performance, since fewer unnecessary messages are being passed. Interconnection graphs are mostly chosen to be sparse, enabling maximum edge efficiency.

 \section{Methodology}
\label{sec:meth}
For our experiments, we use the CLRS framework for neural algorithmic reasoning \cite{velickovic_clrs_2022}. The default hidden size is 128, but we include experiments with smaller sizes in appendix \ref{sec:app}. The train data samples have input sizes 4, 7, 11, 13 and 16, while testing is done on input size $n=64$.\footnote{Earlier versions of this paper report performance when training only on samples of size 16.}

\subsection{The CLRS Framework}
CLRS follows the encode-process-decode paradigm. After encoding the input, a recurrent GNN denoted as \emph{processor} network carries out the main computation, until finally its output is routed through the decoder network. To help performance and distinguish between different algorithms solving the same task, not only the final output is evaluated, but also the intermediate states. The ground truths of these are referred to as \emph{hints}. For further details, we kindly refer the reader to \cite{velickovic_clrs_2022}. 

\subsection{Considered Algorithms}
To test the hypothesis, we consider the two elementary tasks of searching and sorting, as well as computing SCC as an example of a graph algorithm. The parallel algorithms are chosen from section \ref{sec:algo}; as sequential counterparts, we use binary search, bubble sort and Kosaraju's SSC algorithm from the CLRS-30 benchmark \cite{velickovic_clrs_2022}. Key data of the GNN we use are listed in table \ref{tab:key}.
We compare performances across various processor networks, namely the wide-spread architectures of DeepSets \cite{zaheer_deep_2017}, GAT \cite{velickovic_graph_2018}, MPNN \cite{gilmer2017neural}, and PGN \cite{velickovic_pointer_2020}. The trajectories of the new algorithms are encoded for the CLRS framework as follows below. Note that in every case, randomized positional information, as proposed by \citet{mahdavi_towards_2023} and standard on CLRS, is provided as part of the input, to emulate the situation of uniquely identified processors.
\begin{table}
\centering
\caption{Worst case asymptotic capacity $c$, node efficiency $\eta$ and edge efficiency $\epsilon$ in our experiments. }
    \begin{tabular}{lcccccc}
    \toprule
         &   \multicolumn{2}{c}{\textbf{Searching}}  &  \multicolumn{2}{c}{\textbf{Sorting}} & \multicolumn{2}{c}{\textbf{SCC}}  \\
          & Seq. & Par. & Seq. & Par. & Seq. & Par. \\
        \midrule
         $c$ & $n \log n$ & $n$ & $n^3$ & $n^2$ & $n(n+m)$ & $n^3$\\ 
         $\eta$ & $n^{-1}$ & 1 & $n^{-1}$ & 1 & $n^{-1}$ & $n^{-1}$\\
         $\epsilon$ & $n^{-2}$ & $1$ & $n^{-2}$ & $n^{-1}$ & $m^{-1}$ & $m^{-1}$\\
         \bottomrule
    \end{tabular}
    \vskip -0.1in
    \label{tab:key}
\end{table}

 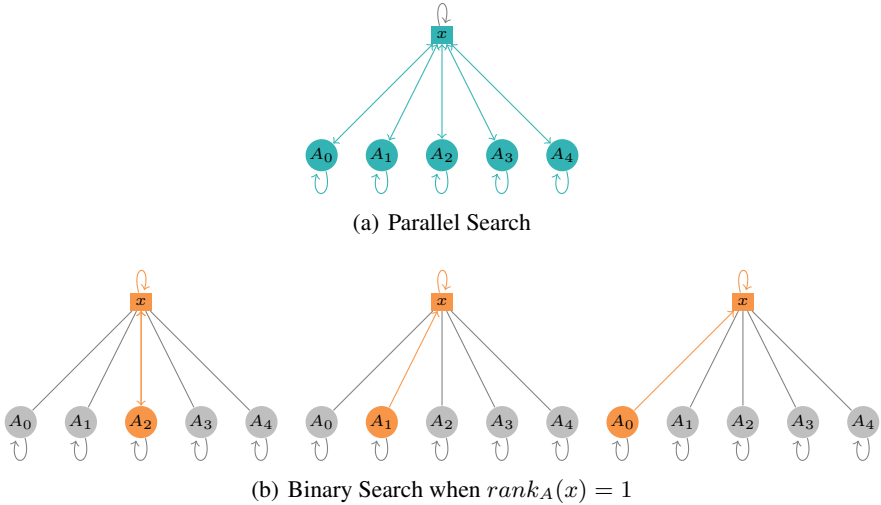
\begin{figure}
 \centering
\subfigure[Parallel Search]{
\begin{tikzpicture}[scale=0.8]
\tikzset{every node/.style={fill=cpar,minimum size=2pt, inner sep=0.5pt}}
\node (k) [inner sep=2pt] at (2,2){\tiny \(x\)}; 
\path (k) [loop above,gray] edge (k); 

\foreach \x in {0,...,4}
	\node [circle] (\x) at (\x,0){\tiny \( A_\x\)}; 

\foreach \x in {0,...,4}
	\draw [cpar, <->] (\x) -- (k);
	
\foreach \x in {0,...,4}
	\path (\x) [cpar, loop below] edge (\x);

\end{tikzpicture}
\label{subfig:par}
} \vfill
\subfigure[Binary Search when $rank_A(x) = 1$]{
\begin{tikzpicture}[scale=0.8]

 \node [fill=cseq, inner sep=2pt] (k0) at (2,2){\tiny \(x\)};
 \path (k0) [loop above,cseq] edge (k0); 
	\foreach \x in {2}{
		\node (0\x) [circle, fill=cseq, inner sep=0.5pt] at (\x, 0){\tiny \( A_\x\)};	
        \draw [cseq, <->] (k0) -- (0\x);
		\path (0\x) [loop below, gray] edge (0\x);
	} 
	\foreach \x in {0,1,3,4}{
		\node (0\x) [circle, fill=gray!50, inner sep=0.5pt] at (\x, 0){\tiny \( A_\x\)};	
        \draw [gray] (k0) -- (0\x);
        \path (0\x) [loop below, gray] edge (0\x);
        }
 
  \node [fill=cseq, inner sep=2pt] (k1) at (7,2){\tiny \(x\)};
  \path (k1) [loop above,cseq] edge (k1); 
	\foreach \x in {1}{
        \pgfmathtruncatemacro{\xn}{\x +5}
		\node (1\x) [circle, fill=cseq, inner sep=0.5pt] at (\xn, 0){\tiny \( A_\x\)};	
		\path (1\x) [loop below, gray] edge (1\x);
	} 
	\foreach \x in {0,2,3,4}{
        \pgfmathtruncatemacro{\xn}{\x +5}
		\node (1\x) [circle, fill=gray!50, inner sep=0.5pt] at (\xn, 0){\tiny \( A_\x\)};	
		\path (1\x) [loop below, gray] edge (1\x);
        \draw [gray] (k1) -- (1\x);}

  \node [fill=cseq, inner sep=2pt] (k2) at (12,2){\tiny \(x\)};
  \path (k2) [loop above,cseq] edge (k2); 
	\foreach \x in {0}{
        \pgfmathtruncatemacro{\xn}{\x +10}
		\node (2\x) [circle, fill=cseq, inner sep=0.5pt] at (\xn, 0){\tiny \( A_\x\)};	
		\path (2\x) [loop below, gray] edge (2\x);
	} 
	\foreach \x in {4,1,2,3}{
        \pgfmathtruncatemacro{\xn}{\x +10}
		\node (2\x) [circle, fill=gray!50, inner sep=0.5pt] at (\xn, 0){\tiny \( A_\x\)};	
		\path (2\x) [loop below, gray] edge (2\x);
        \draw [gray] (k2) -- (2\x);}

\draw [->, cseq] (02)--(k0); 
\draw [->, cseq] (20)--(k2); 
\draw [->, cseq] (11)--(k1); 
\end{tikzpicture}
\label{subfig:seq}
}

\caption{Necessary information flow when searching $x$ in $A = [A_0, \dots, A_4]$ using different algorithms. Active nodes and edges in color.}
\label{fig:updateStuc}
\end{figure}

\subsubsection{Searching}
\paragraph{Parallel Search.}

The hints for parallel search of $x$ in $A$ closely resemble its template. As to be seen in figure \ref{subfig:par}, each item $A_i$ of $A$ is represented by one node of an empty graph. A node \texttt{mask} indicates whether $A_i \leq x$. The position $rank_A (x)$ of $x$ in $A$ is predicted by the graph feature as a categorical variable over the nodes (\texttt{pointer} in \cite{velickovic_clrs_2022}). Therefore we introduce an extra node carrying $x$ as a placeholder to allow for as many categories as possible positions of $x$.

To perfectly predict the outcome in this setting, the graph nodes may be updated by
\begin{align*}
    h_i = ReLU (A_i -x),
\end{align*}
yielding $h_i = 0$ if and only if $A_i \leq x$.

So the graph feature may be computed by
\begin{align*}
    rank_A (x) = \min \{i=1,\dots,n : h_i = 0 \}. 
\end{align*}
These steps closely align with the considered neural update functions, especially since the function updating the graph level possesses its own set of parameters. Additionally, the roll-out has a constant length, leaving room for only a constant number of redundant edges, see figure \ref{subfig:par} and table \ref{tab:key}. Altogether, we expect high performance on parallel search.

 \paragraph{Binary Search.} Opposed to parallel search, binary search has an optimal complexity of $O(\log n)$. But given the need for $n$ nodes, it still requires an enhanced capacity of $O(n \log n)$, yielding low node efficiency. In CLRS-30, binary search is executed on a complete graph (whose edges are omitted in figure \ref{subfig:seq} to avoid clutter), impairing edge efficiency, see table \ref{tab:key}. Low efficiency is visible in figure \ref{subfig:seq} by the amount of grey components. 
 
\subsubsection{Sorting}
\paragraph{OETS}
Swapping the \texttt{scalar} items would require making numerical predictions. Instead, we predict changing predecessors as \texttt{pointers}, following preimplemented examples. To still provide edges between nodes holding items to compare, we have to operate on a complete graph, sacrificing edge efficiency (see table \ref{tab:key}), since only $\Theta(n)$ edges are active in each round, so $\epsilon = n/n^2$. As hints, we feed for each round the current predecessors along with an edge \texttt{mask} indicating whether two nodes have to switch their role, and a graph-level \texttt{mask} with the parity of the round, serving as a rudimentary clock.
\paragraph{Bubble Sort.}
Though Bubble Sort induces the same amount of operations $O(n^2)$ as OETS, it requires a larger network to be executed on (table \ref{tab:key}). Again, along with operating over a complete graph, this entails low efficiencies.

\subsubsection{Strongly Connected Components}
\paragraph{DCSC}
\label{sec:SCC}
We input the undirected adjacency matrix as edge \texttt{mask}, along with the directed one as \texttt{scalar}. Parallelizing the recursive calls of DCSC on multiple disjoint sets would require an extra feature dimension for every search that is going on. Therefore we only let the two BFS starting from the same source node be executed in parallel, which we each encode as is standard in CLRS-30. Additionally, a binary \texttt{mask} on each node is flipped to 1 as soon as it is discovered from both directions, indicating it belongs to the currently constructed SCC (this is reset at the start of every new search). At the same time, it receives a \texttt{pointer} to the source, which in the end constitutes the output. Throughout, we keep track of undiscovered nodes in another node \texttt{mask}. We choose the node with the smallest index from this set as the next source.

\begin{figure}
\centering
     \begin{tikzpicture}[scale=0.6, Dot/.style={circle,fill=cpar,inner sep=2pt,,minimum size=0.5cm}] 
            \node [left of = j]{};
            \node (i) [Dot] {\tiny \(\infty\)};
            \node (agg) [black,left of=i, circle, inner sep=0, outer sep=0] {min};
            \draw[->, cpar] (agg) -- (i);
            \node [left of=agg, Dot, inner sep=3pt] (j) {\small \(s\)};
            \draw[->, cpar] (j) -- (agg);
            \node [above of=j, Dot] (k) {\tiny \(\infty\)};
            \draw[->, cpar] (k) -- (agg);
            \node [below of=j, Dot] (l) {\tiny \(\infty\)};
            \draw[->, cpar] (l) -- (agg);
            \node [left of = j]{};

            \node (space)[right of=i]{};
            \node [right of=space, Dot, inner sep=3pt] (j) {\small \(s\)};
            \node (agg) [right of=j, circle, inner sep=0] {min};
            \node (i) [right of = agg, Dot, inner sep=3pt] {\small \(s\)};
            \draw[->, cpar] (agg) -- (i);
            \draw[->, cpar] (j) -- (agg);
            \node [above of=j, Dot] (k) {\tiny \(\infty\)};
            \draw[->, cpar] (k) -- (agg);
            \node [below of=j, Dot, inner sep=3pt] (l) {\small \(s\)};
            \draw[->, cpar] (l) -- (agg);
        \end{tikzpicture}
        \caption{Consecutive steps of passing the source node index $s$ during a BFS of DCSC. Note how repeating the computation would not change the state of the rightmost node, so redundant computations do not require to be parameterised differently.} 
        \label{fig:BFS}
\end{figure}
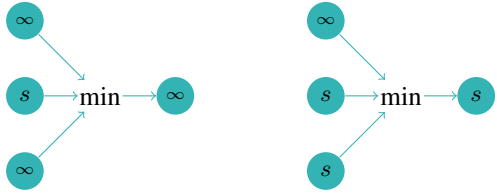

DCSC spends most of its time on the repeated BFS, a subroutine known to be learned well even on relatively simple architectures \cite{velickovic_neural_2020}, as it aligns well with neural execution \cite{dudzik_graph_2022}. 
Note how they let each node consider all its incoming edges in parallel, as is done on CLRS-30. This not only allows the trajectory to be shortened from $O(n+m)$ to $O(n)$ but also prevents redundant computations from having to be handled explicitly. 
Except for the source, each node can carry out the same computation at each step (see \cite{velickovic_neural_2020} for details) -- just that this will only change its state whenever information flowing from the start node reaches it. DCSC only has to pass the index $s$ of the source node instead of computing predecessor pointers, so computation looks like depicted in figure \ref{fig:BFS}, closely resembling the situation in figure \ref{subfig:NN}. Therefore, efficiency is expected to be less important for predictive performance in this special case. An obvious upper bound to DCSC's run time is $O(n^2)$, accounting for one (two-sided) BFS per node, resulting in the big capacity reported in table \ref{tab:key}. There is also no guarantee for more than one node and edge being active per step per BFS, resulting in low efficiencies. But this represents edge cases at best, such that the average trajectories will be much shorter and more efficient, as experiments will show. The core of DCSC aligning so well with neural execution promises good results.

\paragraph{Kosaraju.}
The skeleton of Kosaraju's algorithm as implemented in CLRS-30, on the other hand, is formed by a depth-first search (DFS), which is more challenging for neural executioners \cite{velickovic_clrs_2022}. As opposed to the closely related BFS, it is hard to parallelize. When relying on lexicographic ordering for tie-braking, it is considered an \emph{inherently sequential} algorithm \cite{reif_depth-first_1985}. Since nodes have to wait for the search to retract from its siblings, the computation cannot be carried out as in figure \ref{fig:BFS}, so processing needs to be timed correctly. The total run time is $O(n+m)$, entailing the capacity and efficiencies reported in table \ref{tab:key}.

\section{Results}
\label{sec:exp}

\begin{table}
	\centering
	\caption{Out-of-distribution micro-F1 scores after 2000 iterations of training sequential versus parallel algorithms on different processor networks, averaged over 3 seeds.}
    \label{tab:results}
    \begin{small}
    \begin{tabular}{lcccccc}
    \toprule
         &   \multicolumn{2}{c}{\textbf{Searching}}  &  \multicolumn{2}{c}{\textbf{Sorting}} & \multicolumn{2}{c}{\textbf{SCC}}  \\
         Arch. & Sequential & Parallel & Sequential & Parallel & Sequential & Parallel \\
        \midrule
         DeepSets & 67.2\%$\pm$10.2 &   \textbf{100\%$\pm$ 0.0}  &   57.5\%$\pm$4.5 &  \textbf{77.8\%$\pm$5.3} & 26.2\%$\pm$8.7 &  \textbf{41.1\%$\pm$14.4} \\
         GAT & 3.4\%$\pm$1.2 &  \textbf{100\%$\pm$0.0} &  28.6\%$\pm$ 13.6 & \textbf{34.3\%$\pm$20.0} & 28.9\%$\pm$2.5&  \textbf{76.6\%$\pm$4.9} \\
         MPNN & 79.8\%$\pm$5.8  &  \textbf{100\%$\pm$ 0.0} &  34.5 \%$\pm$9.3 & \textbf{52.4\%$\pm$23.4}  & 34.0\%$\pm$1.1 &  \textbf{75.0\%$\pm$6.2} \\
         PGN & 77.2\%$\pm$9.2 &  \textbf{100\%$\pm$ 0.0}  &  50.9 \%$\pm$17.8  & \textbf{62.0\%$\pm$25.9} & 35.1\%$\pm$3.1 &  \textbf{82.3\%$\pm$2.2} \\
         \bottomrule
    \end{tabular}
    \end{small}
\end{table}

\begin{figure}
    \centering
    \includegraphics[scale=0.4]{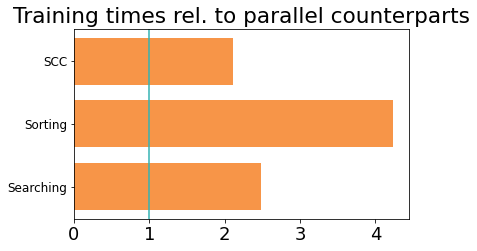}
    \caption{Training times of sequential algorithms with samples of input size $n=16$, relative to their respective parallel counterparts.}
    \label{fig:time}
\end{figure}

\begin{figure}
    \centering
    \subfigure[Searching]{\includegraphics[scale=0.3]{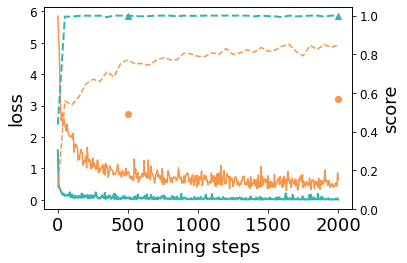}}
    \subfigure[Sorting]{\includegraphics[scale=0.3]{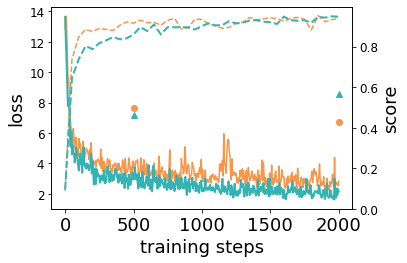}}
    \subfigure[SCC]{\includegraphics[scale=0.3]{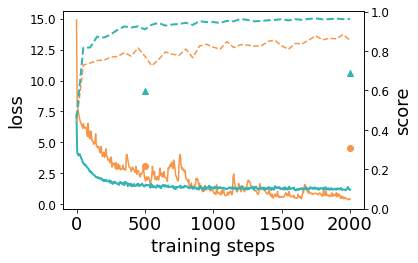}}
    \caption{Losses (solid lines) and validation scores (dashed lines) over time on different tasks. Performance on sequential algorithms in orange, on parallel ones slightly thicker in turquoise. Test scores after 500 and 2000 steps as orange points and turquoise triangles for sequential and parallel algorithms, respectively.}
    \label{fig:valloss}
\end{figure}

Predictive performance is reported in table \ref{tab:results}. Parallel search achieves perfect results and converges very quickly (see figure \ref{fig:valloss}). Meanwhile, training time on sample size 16 is reduced by a factor of more than 2 as compared to binary search (see figure \ref{fig:time}). Parallel sorting outperforms its sequential opponent as well, though performance on both is subject to big standard deviations. Despite both algorithms requiring the same asymptotic number of operations, training OETS takes less than a quarter of the time needed for bubble sort (figure ~\ref{fig:time}. Despite DCSC's only partial parallelization and the asymptotically optimal linear run time of its sequential opponent, training time is more than halved for the SCC task as well. At the same time, predictions become up to more than twice as accurate.

\section{Discussion}
\label{sec:diss}
Neural efficiency only loosely correlates with predictive performance when comparing tables \ref{tab:key} and \ref{tab:results}. This is not too surprising, since correctly parameterising redundant computations is only one of many aspects that make a function hard to learn. We propose a rather one-sided relationship, where low efficiencies can harm accuracy (if not circumvented as in BFS, see section \ref{sec:SCC}), but high efficiencies do not necessarily enhance learning success. 

We would like to highlight the importance of taking the perspective on neural networks as computational models when executing algorithms, as it opens access to the rich theory of computational complexity. E.g. the classes of NC (efficiently parallelizable) and P-complete problems (mostly thought of as inherently sequential) \cite{greenlaw_limits_1995} inform us on which tasks may be hard to execute neurally, to tackle them more effectively. However, in doing so, it is important to keep in mind the gap between the respective sets of constant time operations, with none being strictly more powerful than the other. On the one hand, a single RAM instruction may need to be approximated by entire subnetworks. On the other hand, one neural step suffices to process all incoming edges of a node during the execution of BFS \cite{velickovic_neural_2020}. This breaks up the strict correspondence between time-processor product and capacity.

\section{Conclusion}
As suggested in section \ref{sec:theory}, parallel algorithms prove to be a lot more efficient to learn and execute on neural architectures than sequential ones. Often, OOD predictions on algorithmic tasks are significantly improved as well, suggesting that higher node and edge efficiency can help learning. Future work has to show how performance is impacted for other tasks, on more elaborate architectures like in \cite{bevilacqua_neural_2023, ibarz_generalist_2022}, and in generalist settings.

\section*{Author Contributions}
\begin{tabular}{ll}
  Valerie Engelmayer:   &  Conceptualization, Formal Analysis, Investigation, Methodology, Software,\\ & Visualization, Writing – original draft \\
  Dobrik Georgiev:    & Resources, Validation\\
  Petar Veličković: & Supervision, Writing – review \& editing
\end{tabular}

\section*{Acknowledgements}
We would like to thank Razvan Pascanu and Karl Tuyls for their valuable comments, as well as Pietro Liò for insightful discussions and Torben Hagerup for the support he provided. 

\bibliographystyle{unsrtnat}

\bibliography{GNN}

\appendix
\section{Appendix}

\label{sec:app}
Better alignment of parallel algorithms may enhance performance on smaller processor networks. Indeed, we observe that decreasing the hidden size from 128 to 32 or even 8 is mostly slightly less impeding in the parallel setting, especially the searching task, see tables \ref{tab:results8} and \ref{tab:results32}, along with figure \ref{fig:ValSmall}.

\begin{table}
	\centering
	\caption{Out-of-distribution micro-F1 scores after 2000 iterations of training sequential versus parallel algorithms on different processor networks with hidden size 8, averaged over 3 seeds.}
    \label{tab:results8}
    \begin{small}
    \begin{tabular}{lcccccc}
    \toprule
         &   \multicolumn{2}{c}{\textbf{Searching}}  &  \multicolumn{2}{c}{\textbf{Sorting}} & \multicolumn{2}{c}{\textbf{SCC}}  \\
         Arch. & Sequential & Parallel & Sequential & Parallel & Sequential & Parallel \\
        \midrule
         DeepSets & 40.8\%$\pm$9.3 &   \textbf{97.9\%$\pm$ 2.9}  &   18.8\%$\pm$7.2 &  \textbf{43.4\%$\pm$6.8} & 30.5\%$\pm$2.5 &  \textbf{43.4\%$\pm$7.2} \\
         GAT & 4.0\%$\pm$0.7 &  \textbf{100\%$\pm$0.0} &  29.9\%$\pm$ 11.1 & \textbf{36.0 \%$\pm$16.3} & 33.5\%$\pm$0.7 &  \textbf{48.2\%$\pm$2.8} \\
         MPNN & 31.9\%$\pm$13.6  &  \textbf{99.0\%$\pm$ 1.5} &  \textbf{40.8 \%$\pm$24.5} & 30.6\%$\pm$18.5  & 29.2\%$\pm$4.3 &  \textbf{46.2\%$\pm$2.2} \\
         PGN & 36.4\%$\pm$20.0 &  \textbf{98.0\%$\pm$ 2.9}  &  41.4 \%$\pm$25.8  & \textbf{51.0\%$\pm$15.6} & 30.7\%$\pm$3.1 &  \textbf{45.9\%$\pm$3.1} \\
         \bottomrule
    \end{tabular}
    \end{small}
\end{table}

\begin{table}
	\centering
	\caption{Out-of-distribution micro-F1 scores after 2000 iterations of training sequential versus parallel algorithms on different processor networks with hidden size 32, averaged over 3 seeds.}
    \label{tab:results32}
    \begin{small}
    \begin{tabular}{lcccccc}
    \toprule
         &   \multicolumn{2}{c}{\textbf{Searching}}  &  \multicolumn{2}{c}{\textbf{Sorting}} & \multicolumn{2}{c}{\textbf{SCC}}  \\
         Arch. & Sequential & Parallel & Sequential & Parallel & Sequential & Parallel \\
        \midrule
         DeepSets & 65.1\%$\pm$7.5 &   \textbf{100\%$\pm$ 0.0}  &  37.0\%$\pm$3.2 &  \textbf{64.3\%$\pm$2.1} & 15.2\%$\pm$4.9 &  \textbf{45.0\%$\pm$5.5} \\
         GAT & 7.5\%$\pm$1.6 &  \textbf{100\%$\pm$0.0} &  43.6\%$\pm$ 9.0 & \textbf{64.6\%$\pm$7.2} & 14.8\%$\pm$12.0&  \textbf{51.1\%$\pm$6.1} \\
         MPNN & 63.1\%$\pm$10.9  &  \textbf{100\%$\pm$ 0.0} &  57.6 \%$\pm$4.8 & \textbf{60.2\%$\pm$9.4}  & 26.2\%$\pm$4.2 &  \textbf{62.7\%$\pm$6.9} \\
         PGN & 77.1\%$\pm$2.5 &  \textbf{100\%$\pm$ 0.0}  &  44.0 \%$\pm$8.3  & \textbf{74.1\%$\pm$5.6} & 21.4\%$\pm$4.3 &  \textbf{66.9\%$\pm$5.9} \\
         \bottomrule
    \end{tabular}
    \end{small}
\end{table}

\begin{figure}
    \centering
    \subfigure[Searching with hidden size 8]{\includegraphics[scale=0.4]{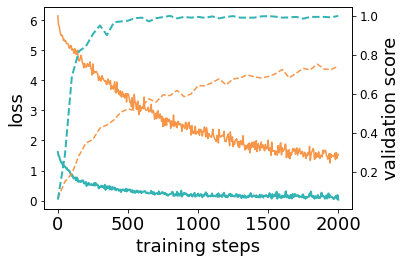}}
    \subfigure[Searching with hidden size 32]{\includegraphics[scale=0.4]{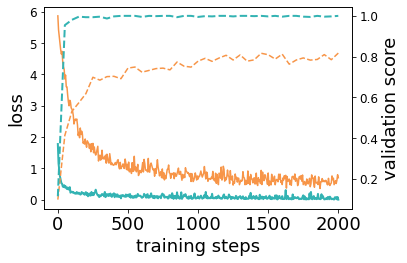}}
    \subfigure[Sorting with hidden size 8]{\includegraphics[scale=0.4]{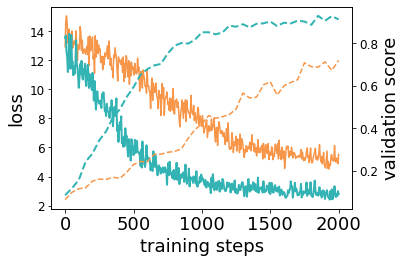}}
    \subfigure[Sorting with hidden size 32]{\includegraphics[scale=0.4]{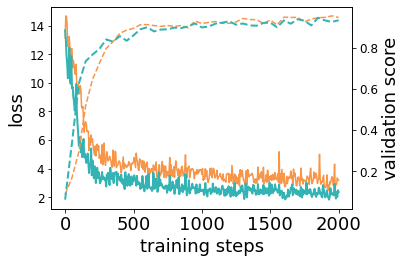}}
    \subfigure[SCC with hidden size 8]{\includegraphics[scale=0.4]{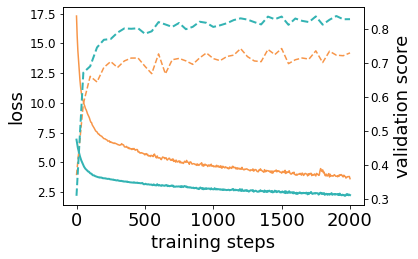}}
    \subfigure[SCC with hidden size 32]{\includegraphics[scale=0.4]{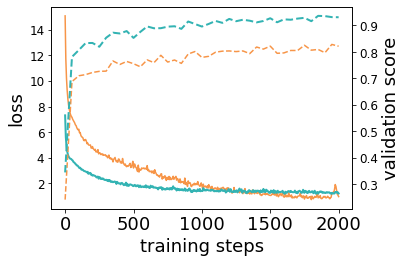}}
    \caption{Losses (solid lines) and validation scores (dashed lines) over time on different tasks with different hidden sizes. Performance on sequential algorithms in orange, on parallel ones slightly thicker in turquoise.}
    \label{fig:ValSmall}
\end{figure}

\end{document}